\documentclass[final, 12pt]{alt2024Arxiv} 

\usepackage{xcolor}
\usepackage[suppress]{color-edits}
\usepackage{bbm}
\usepackage{hyperref}
\usepackage{cleveref}
\usepackage{physics}
\usepackage{nicefrac}

\newcommand{\cR}{\mathbb{R}}
\newcommand{\cX}{\mathcal{X}}

\newcommand{\cD}{\mathcal{D}}
\newcommand{\cT}{\mathcal{T}}
\newcommand{\cC}{\mathcal{C}}

\newcommand{\bbR}{\mathbb{R}}

\newcommand{\IGNORE}[1]{}

\newcommand{\ind}[1]{\mathbbm{1}\{#1\}}
\renewcommand{\set}[1]{\{ #1 \}}

\newcommand{\len}{\mathtt{len}}

\newcommand{\algmq}{\mathtt{MidpointQuery}}
\newcommand{\poly}{\mathtt{poly}}

\newtheorem*{theorem*}{Theorem}
\newcommand{\tildeO}{\widetilde{O}}
\renewcommand{\ip}[2]{\langle #1, #2\rangle}
\usepackage{multirow}
\usepackage{thmtools, thm-restate}

\addauthor{sz}{red}
\addauthor{sz}{red}
\addauthor{zz}{red}
\addauthor{lb}{cyan}
\addauthor{draft}{blue}
\addauthor{todo}{purple}

\title[Corruption-Robust Lipschitz Contextual Search]{Corruption-Robust Lipschitz Contextual Search}
\usepackage{times}

\altauthor{%
 \Name{Shiliang Zuo} \Email{szuo3@illinois.edu}\\
 \addr University of Illinois Urbana-Champaign
}

\begin{document}

\maketitle

\begin{abstract}%
  I study the problem of learning a Lipschitz function with corrupted binary signals. The learner tries to learn a $L$-Lipschitz function $f: [0,1]^d \rightarrow [0, L]$ that the adversary chooses. There is a total of $T$ rounds. In each round $t$, the adversary selects a context vector $x_t$ in the input space, and the learner makes a guess to the true function value $f(x_t)$ and receives a binary signal indicating whether the guess is high or low. In a total of $C$ rounds, the signal may be corrupted, though the value of $C$ is \emph{unknown} to the learner. The learner's goal is to incur a small cumulative loss. This work introduces the new algorithmic technique \emph{agnostic checking} as well as new analysis techniques. I design algorithms which: for the absolute loss, the learner achieves regret $L\cdot O(C\log T)$ when $d = 1$ and $L\cdot O_d(C\log T + T^{(d-1)/d})$ when $d > 1$; for the pricing loss, the learner achieves regret $L\cdot \widetilde{O} (T^{d/(d+1)} + C\cdot T^{1/(d+1)})$. 

\szdelete{for the pricing loss with a learner's choice of parameter $u \in (0, 1)$, the learner achieves $\widetilde{O}(T^{\frac{(1-u)d + u}{(1-u)d + 1} } + T^u C^{1-u} + C^{1/u} )$ loss. }
\end{abstract}


\section{Introduction}
Consider the dynamic pricing problem, where a seller attempts to sell a product to a buyer each day without any prior knowledge of the maximum price $u$ the buyer is willing to pay (i.e. the buyer's private valuation for the product). The seller must learn this price by observing the buyer's behavior, which takes the form of a binary signal: a purchased or an unpurchased item. An unpurchased item indicates an overpriced product, leading to a possible loss of the entire revenue $u$; a purchased item indicates an underpriced product, leading to a possible loss of surplus (the difference between $u$ and the posted price). However, the seller does not directly observe the revenue loss and only observes the binary signal of the buyer. The seller adjusts the posted price according to the signal and gradually converges to the optimal price. An early important work by~\cite{kleinberg2003value} studied dynamic pricing from a regret-minimization perspective and fully characterize the optimal regret of the seller for the most basic form of this problem.

The \emph{contextual search} (or \emph{contextual pricing}) problem~(\cite{liu2021optimal, leme2022contextual, mao2018contextual}) is motivated by the fact that buyers have different valuations for differentiated products with different features. Henceforth the features of products are referred to as \emph{context} vectors. A common assumption is that the buyer's valuation is a linear function on the context vector~(\cite{liu2021optimal, leme2022contextual}). In this setting,~\cite{liu2021optimal} obtains a nearly optimal regret bound. Another assumption is that the valuation function is Lipschitz with respect to the context vector. This setting was studied by~\cite{mao2018contextual}, in which they gave nearly optimal algorithms based on discretization of the input space.
\lbcomment{Another assumption that ... input space. -> ~\cite{mao2018contextual} studied the assumption of Lipschitzness: the valuation function only requires Lipschitzness in the context vectors without assuming an explicit parametric form and gave nearly optimal algorithms based on discretization of the input space}

\szcomment{Alternate version. The contextual search problem comes from the idea that various product features influence a buyer's valuation. These product features are referred to as context vectors. A common assumption is that the buyer's valuation is a linear function on the context vectors. Liu et al.~\cite{liu2021optimal} have almost completely characterized the optimal regret in this setting. Lipschitzness is another approach: the valuation function only requires Lipschitzness in the context vectors without assuming an explicit parametric form~\cite{mao2018contextual}. }

In practice, it is unreasonable to assume the signals the seller receives are perfect. For example, buyers can act irrationally, or in extreme cases may even exhibit malicious behaviors, resulting in faulty signals. \lbcomment{could give some specific examples here} This motivates the study of designing contextual search algorithms robust to corruptions in the binary signals that the learner observes. \szdelete{The corruptions are modeled as generated by an adversary who has the power to corrupt the binary signals observed by the learner.}\lbcomment{This sounds weird grammatically} Previous works~(\cite{krishnamurthy2022contextual, leme2022corruption}) studied \emph{linear} contextual search with corruptions, in which an adversary has the power to corrupt the binary signals. This work studies \emph{Lipschitz} contextual search with corruptions and proposes corruption-robust algorithms. 

\szcomment{This work considers an adversary model where the adversary has the power to directly corrupt the binary signal after observing the learner's guess. This adversary model is much stronger than the recent works~\cite{krishnamurthy2022contextual, leme2022corruption} on corruption-robust linear contextual search. }\szcomment{studied linear contextual search in an \emph{adversarial corruption} model. In their model, the adversary commits to a corruption $z_t$ to be added to the function value \emph{before} observing the guess by the learner. In the current work, \lbcomment{In the current work -> In this work} the adversary is able to directly decide whether to corrupt the signal, hence the power of the adversary in this work is significantly stronger. }

\paragraph{Loss Functions and Applications. } The literature on contextual search usually considers two loss functions, the \emph{pricing loss} and the \emph{absolute loss} (e.g \cite{mao2018contextual}, \cite{krishnamurthy2022contextual}). The pricing loss is defined by
\[
\ell(q, f(x)) = f(x) - q\cdot \ind{q\le f(x)}. 
\]
The pricing loss captures the dynamic pricing setting described above. The optimal price for a product with context vector $x$ is $f(x)$. If the seller overprices $q > f(x)$, she loses the entire revenue $f(x)$; if the seller underprices $q \le f(x)$, she loses the potential surplus $f(x) - q$. The cumulative pricing loss then measures the regret of the seller had she known the valuation function $f$. It should be noted the actual loss of each round is never revealed to the learner, and only the binary signal is revealed. 

The absolute loss is defined as
\[
\ell(q, f(x)) = \abs{q - f(x)}. 
\]
An application of absolute loss is personalized medicine. Consider a healthcare provider experimenting with a new drug. Assume the optimal dosage for a patient with context feature $x$ is $f(x)$, and the given dosage is $q$. The absolute loss then measures the distance between the injected dosage and the optimal dosage. The learner (healthcare provider) is not able to directly observe this loss but can observe a binary signal informing whether she overdosed or underdosed. 

For either the pricing loss or absolute loss $\ell$, the cumulative loss (or regret, this work will use the two terms interchangebly) of the learner after $T$ rounds is then defined as:
\[
\sum_{t\in [T]} \ell_t = \sum_{t\in [T]} \ell(q_t, f(x_t)). 
\]

\subsection{Contributions}
\renewcommand{\arraystretch}{1.1}
\begin{table}[h]
\centering
\begin{tabular}{|c|c|c|c|}
\hline
                         Setting  & Adversary & Absolute Loss                                                                                             & Pricing Loss                                                                                  \\ \hline
\multirow{2}{*}{Linear}    & No        & \begin{tabular}[c]{@{}c@{}}$O(d\log d)$\\ \cite{liu2021optimal}\end{tabular}                                & \begin{tabular}[c]{@{}c@{}}$O(d\log\log T + d\log d)$\\ \cite{liu2021optimal}\end{tabular}     \\ \cline{2-4} 
                           & Weak      & \begin{tabular}[c]{@{}c@{}}$O(C+d\log T)$\\ \cite{leme2022corruption}\end{tabular}                              & \begin{tabular}[c]{@{}c@{}}$O(C d^3 \poly\log(T))$\\ \cite{krishnamurthy2022contextual}\end{tabular} \\ \hline
\multirow{2}{*}{Lipschitz} & No        & \begin{tabular}[c]{@{}c@{}}$O(T^{(d-1)/d})$\\ $O(\log T)$ when $d = 1$\\ \cite{mao2018contextual}\end{tabular} & \begin{tabular}[c]{@{}c@{}}$O(T^{d/(d+1)})$\\ \cite{mao2018contextual}\end{tabular}               \\ \cline{2-4} 
                           & Strong    & \begin{tabular}[c]{@{}c@{}}$O(C\log T + T^{(d-1)/d})$\\ $O(C\log T)$ when $d = 1$\\ {(This work)}\end{tabular} & \begin{tabular}[c]{@{}c@{}}$\tildeO(C\cdot T^{1/(d+1)} + T^{d/(d+1)})$\\ {(This work)}\end{tabular}                          \\ \hline
\end{tabular}
\vspace{0.1in}
\caption{Regret for contextual search under different settings. For the Lipschitz contextual search problem, the Lipschitz $L$ is treated as constant and omitted in the regret. See Remark~\ref{remark:adversaryModel} for the definition of weak / strong adversary. }
\label{table:results}
\end{table}
In this work, I design corruption-robust algorithms for learning Lipschitz functions with binary signals. The learner tries to learn a Lipschitz function $f$ chosen by an adversary. The domain of $f$ is the $d$-dimensional cube $[0,1]^d$. At each round, the adversary presents the learner with an adversarially chosen context vector $x_t \in [0,1]^d$, and the learner submits a guess $q_t$ to the value of $f(x_t)$ and observes a binary signal indicating whether the guess is low or high. Throughout the course of $T$ rounds, a total of $C$ signals may be corrupted, though the value of $C$ is unknown to the learner. The goal of the learner will be to incur a small cumulative loss that degrades gracefully as $C$ increases. 

I design corruption-robust algorithms under two loss functions, absolute loss and pricing loss. The corruption-robust algorithm for absolute loss is given in~\Cref{sec:symmetric}, and the algorithm for pricing loss is given in~\Cref{sec:pricing}. The main results and comparison with the most relevant works are summarized in \Cref{table:results}. I also give lower bounds showing the upper bounds in this work are tight. 





\lbcomment{consider making the following paragraphs as 1.3, a new subsection. it makes sense to me to introduce related works first and then compare your work to related works}
This work introduces a new algorithmic technique of \emph{agnostic checking} and shows how \emph{agnostic checking} can be incorporated into the adaptive discretization and uniform discretization procedure, usually used in learning Lipschitz functions. At a very high level, the input space is discretized into ``bins", and the learner maintains an associated range for each bin which serves as an estimate for the image of $f$ in the bin. The associated range becomes more accurate as the learner submits more queries and observes more binary signals. However, since the binary signals may be corrupted, the associated range may not be accurate. The learner performs checking steps by querying the boundaries of the associated range and tries to ensure the associated range is accurate. \lbcomment{delete previous sentence}A key challenge is that checking steps may also be corrupted, hence the new technique is termed ``agnostic checking". New analysis techniques are introduced to bound the regret while the learner is agnostic as to whether the checking steps are corrupted or not. 

\szdelete{
\emph{agnostic checks} ensure the estimated function values in the discretized input space are accurate in the presence of adversarial corruptions. The analysis is much more intricate and requires significantly new ingredients, since \emph{agnostic checks} can also be corrupted (hence the name). }\lbcomment{the first and the last sentences seem to convey the same meaning, but it's not clear what's your message to the readers}

\lbcomment{do you ever mention Table 1 in these paragraphs? It seems like Table 1 is related to comparison of previous works. If Table 1 is used in later paragraph, you should consider moving it to a better place}

\subsection{Related Work}
The two most related threads to this work are contextual search and corruption-robust learning. The linear contextual search problem was studied in~\cite{liu2021optimal, leme2022contextual}. The recent work by \cite{liu2021optimal} achieved state-of-the-art regret bounds. For the Lipschitz contextual search problem, \cite{mao2018contextual} proposed algorithms based on adaptive discretization and binary search. They also showed their algorithm to be nearly optimal. The contextual search problem is also closely related to the dynamic pricing problem, for an overview see~\cite{den2015dynamic}. 

Recently, designing learning algorithms robust to corruption and data-poisoning attacks has received much attention. \cite{leme2022corruption} gave corruption-robust algorithms for the linear contextual problem based on density updates, improving over the work by~\cite{krishnamurthy2022contextual}. The study of adversarial attack and the design of corruption-robust algorithms has also appeared in bandit learning~(\cite{lykouris2018stochastic, gupta2019better, jun2018adversarial, zuo2023near, garcelon2020adversarial, bogunovic2021stochastic}) and reinforcement learning~(\cite{lykouris2021corruption, zhang2022corruption, chen2021improved}), to name a few. \szcomment{TODO, add cite. }

This work is also related to the contextual bandits and bandits in metric spaces~(\cite{kleinberg2008multi, kleinberg2010sharp, cesa2017algorithmic, slivkins2011contextual}). The notable work by \cite{slivkins2011contextual} studied a setting where the learner picks an arm after observing the contextual information each round, and subsequently the context-arm dependent reward (or equivalently loss) is revealed. \szcomment{In this work \lbcomment{In the current work -> In this work}, the loss is not revealed to the learner, and only a binary signal is observed, hence the feedback structure is fundamentally different. \lbcomment{You may want to add a couple of sentences to explain why the difference matter, and perhaps why it matters should go to the end of this paragraph or go to the contribution section to emphasize your novelty; review other people's work first and explain what is unique in your work} }\szcomment{TODO. add cite}

Another interesting direction peripheral to this work is learning with strategic agents. In this setting, the learner is not interacting with an adversary but with a strategic agent whose goal is to maximize his long-term utility.~\cite{amin2013learning} studied designing auctions with strategic buyers; ~\cite{amin2014repeated} and~\cite{drutsa2020optimal} incorporated contextual information and designed contextual auctions with strategic buyers. 

\lbcomment{do you need a paragraph describing what's the novelty of your work? I remember some reviewers are questioning your novelty related to previous work.}

\subsection{Comparison with previous work}
The problem of Lipschitz contextual search was studied in~\cite{mao2018contextual}, in which they proposed near-optimal algorithms when the binary signals are perfect each round (i.e. no corruptions). This work studies the Lipschitz contextual search with adversarial corruptions and designs corruption-robust algorithms while being agnostic to the corruption level $C$. 

\cite{krishnamurthy2022contextual} and~\cite{leme2022corruption} studied the linear contextual search problem with adversarial corruptions. In their work, the underlying unknown function $f$ takes the parametric linear form $f(x) = \ip{\theta}{x}$, where $\theta$ is the unknown parameter and $x$ is the context vector. \cite{krishnamurthy2022contextual} proposed maintaining a knowledge set that contains the true parameter; \cite{leme2022corruption} proposed maintaining a density over the parameter space. This work does not make any parametric assumptions and works with the non-parametric Lipschitz assumption, i.e. $f$ is only assumed to be Lipschitz. Moreover, the adversary model in this work is stronger than the adversary model studied in their work (see Remark~\ref{remark:adversaryModel} below). 

\cite{leme2021learning} studied an algorithmic solution to the non-stationary dynamic pricing problem, which is superficially similar to this work. Their work also used the notion of ``checking steps". This work is different from theirs in the following aspects.\lbcomment{previous sentence can be deleted} The algorithm in this work is designed for an \emph{adversarial corruption} setting with \emph{context} information (under the Lipschitz assumption). Since checking steps can also be corrupted and the learner will not know whether the signal obtained from checking steps is accurate, this work terms these as \emph{agnostic checking}. This work shows entirely new analysis techniques on how agnostic checking can be incorporated with the discretization technique used for learning Lipschitz functions. By contrast,~\cite{leme2021learning} studied a non-stationary dynamic pricing setting without context and when there is no corruption in the binary signals. \lbcomment{``where there is'' is not gramatically correct}

Some recent work also studied corruption-robust algorithms for contextual bandit problems (e.g.~\cite{bogunovic2021stochastic} and~\cite{he2022nearly}). Note that the feedback and reward model in this work is fundamentally different from contextual bandit problems. Specifically, for contextual bandits, the adversary first chooses a function (the reward function), and then the learner chooses some action (based on some context) each round and observes a (possibly corrupted) reward. In this work, the contexts are chosen by the adversary, and the learner guesses the value of the unknown function and only observes a (possibly corrupted) binary signal of whether the guess is low or high.

\section{Problem Formulation}
\label{sec:formulation}
\begin{definition}
Let $f: \cR^d \rightarrow \cR$. If $f$ satisfies:
\[
\abs{f(x) - f(y)} \le L\norm{x - y}_\infty
\]
for any $x, y$, then $f$ is $L$-Lipschitz. 
\end{definition}

An adversary selects an $L$-Lipschitz function $f:[0,1]^d \rightarrow [0,L]$. The function $f$ is initially unknown to the learner. At each round $t$, the interaction protocol proceeds as follows:
\begin{enumerate}
    \item The adversary selects a \emph{context} $x_t$ in the input space $\cX = [0,1]^d$.
    \item The learner observes $x_t$ and submits a guess $q_t$ to the value $f(x_t)$.
    \item The adversary observes $q_t$ and computes $\sigma(q_t - f(x_t))$. Here $\sigma(x)$ is the step function that takes value 1 if $x > 0$ and takes value 0 if $x \le 0$. 
    \item The adversary decides whether to corrupt the signal and sends the (possibly corrupted) signal $\sigma_t$ to the learner. 
    \item The learner observes the signal $\sigma_t$, and suffers an unobservable loss $\ell_t := \ell(q_t, f(x_t))$. 
\end{enumerate}

\szdelete{ the adversary selects a \textit{context} $x_t$ in the input space $\cX$, and the learner submits a guess $q_t$ to the value $f(x_t)$. }

\szdelete{After the learner submits a guess, she observes a binary signal $\sigma_t \in \set{0,1}$. The binary signal $\sigma_t$ at round $t$ may be uncorrupted or corrupted. If the signal is uncorrupted, then
\[
\sigma_t := \sigma(q_t - f(x_t)). 
\]
If the signal is corrupted, then
\[
\sigma_t := 1 - \sigma(q_t - f(x_t))
\]
Here, $\sigma(u)$ is the step function that takes the value 1 if $u > 0$ and 0 if $u \le 0$. Hence an uncorrupted signal of $\sigma_t = 1$ indicates a guess too high, and a signal of $\sigma_t = 0$ indicates a guess too low. 
}

In step 4, it is assumed that throughout the course of $T$ rounds, at most $C$ rounds are corrupted. In uncorrupted rounds, the signal $\sigma_t = \sigma(q_t - f(x_t))$. In corrupted rounds, the signal $\sigma_t = 1 - \sigma(q_t - f(x_t))$. In other words, for uncorrupted rounds, a  signal of $\sigma_t = 0$ indicates a guess too low and a signal of $\sigma_t = 1$ indicates a guess too high, and the adversary flips this signal in corrupted rounds. Critically, the value $C$ is \emph{unknown} to the learner. The algorithms presented in this work are agnostic to the quantity $C$, and the performance degrades gracefully as $C$ increases. \szcomment{TODO. add more detail. Done. }

In step 5, the learner suffers a loss $\ell_t = \ell(q_t, f(x_t))$. Note that the functional form of the loss function is known to the learner but the loss value is never revealed to the learner. The goal of the learner will be to incur as little cumulative loss as possible, where the cumulative loss is defined as
\[ 
\sum_{t=1}^T \ell_t = \sum_{t=1}^T \ell(q_t, f(x_t)). 
\]

Two loss functions are considered in this work, the absolute loss
\[
\ell(q, f(x)) = \abs{f(x) - q}
\]
and the pricing loss
\[
\ell(q, f(x)) = f(x) - q\cdot \ind{q \le f(x)}. 
\]
Separate algorithms are presented for each loss function in the following sections. 

\szcomment{make a comment that linear/lipschitz already different}
\begin{remark}
\label{remark:adversaryModel}
    The adversary model in this work is much stronger compared with previous work on corrupion-robust linear contextual search (\cite{krishnamurthy2022contextual, leme2022corruption}) (apart from the obvious differences in the assumptions on $f$, i.e., this work assumes $f$ to be Lipschitz whereas previous works assumed $f$ to be linear). In previous works, the adversary has to commit to a `corruption level' $z_t$ \emph{before} seeing the guess $q_t$ submitted by the learner. In uncorrupted rounds $z_t = 0$, and in corrupted rounds $z_t$ will be a bounded quantity (e.g. $z_t\in [-1,1]$). Then the learner observes the signal $\sigma_t = \sigma(q_t - (f(x_t) + z_t))$. In other words, the binary signal is generated according to the corrupted function value $f(x_t) + z_t$ instead of the true function value $f(x_t)$. 

    In this work, the adversary's power is significantly increased. The adversary has the power to directly corrupt the binary signal and does so \emph{after} the learner has submitted her guess. However, with this being said, \cite{leme2022corruption} also considered a stronger notion of regret (under the weak adversary), where $C$ is defined as the cumulative corruption level instead of the number of rounds in which corruption occurred. 
\end{remark}

\paragraph{Notations. } For a one-dimensional interval $I$, the notation $\len(I)$ denotes the length of $I$. For a hypercube $I\subset \bbR^d$, the notation $\len(I)$ denotes the length of any side of $I$. In some cases, to highlight the dependence on $T$ and $C$, the notation $O_d(\cdot)$ is used to hide dependence on $d$. 

\szdelete{
\begin{remark}
It should be noted there is a slight difference in the range of $f$ compared with the formulation in~\cite{mao2018contextual}. In~\cite{mao2018contextual}, the range of $f$ is assumed to be $[0,1]$ regardless the value of $L$, where this work assume the range of $f$ to be $[0,L]$. The author believes the assumption in this work is the more natural one, as the range of an $L$-Lipschitz function on $[0,1]^d$ is $L$. Assuming the range to be $[0,L]$ also avoids some of the unnecessary complications (in the author's opinion) that arise from the scaling of $L$ (see~\cite{mao2018contextual} appendix A). If one replaces the range $[0,L]$ to $[0,1]$ for $L\ge 1$, the proposed algorithms in this work lead to sharper regret bounds (see details in following sections). 
\end{remark}
}

\lbcomment{I like these remarks as they compare and contrast with previous works to illustrate your novelty again and allow the readers easily follow}

\section{Algorithm for Absolute Loss}
\label{sec:symmetric}

This section gives a corruption-robust algorithm for the absolute loss, defined by
\[
\ell(q_t, f(x_t)) = \abs{f(x_t) - q_t}. 
\] 
I will first design an algorithm for $d = 1$ in this section (\Cref{algo:1dabsolute}), then extend the algorithm to $d \ge 2$. 

\begin{algorithm2e}
\caption{Learning with corrupted binary signals under absolute loss for $d=1$}
\label{algo:1dabsolute}
Learner maintains a partition of intervals $I_j$ of the input space throughout the learning process\;
For each interval $I_j$ in the partition, learner stores a checking interval $S_j$, and maintains an associated range $Y_j$\;
The partition is initialized as $8$ intervals $I_j$ with length $1/8$ each, and each $I_j$ has associated range $Y_j$ and checking interval $S_j$ set to $[0,L]$\;
\For {$t = 1, 2, ..., T$} {
    Learner receives context $x_t$\;
    Learner finds interval $I_j$ such that $x_t \in I_j$\;
    Let $Y_j$ be the associated range of $I_j$\;
    \If {Exists an endpoint of $S_j$ not yet queried} { 
        Learner selects an unqueried endpoint of $S_j$ as guess \;
        \If {Learner guessed $\min(S_j)$ and $\sigma_t = 1$, or guessed $\max(S_j)$ and $\sigma_t = 0$} {
            Mark $I_j$ as {dubious}\;
        }
        \If {Both endpoints of checking interval $S_j$ have been queried} {
            \If {$I_j$ marked {dubious}} {
                Set range $Y_j := [0,L]$\;\label{algoline:marked_rangeset}
            }
            \Else {
                Set range $Y_j = [\min(S_j) - L\cdot \len(I_j)), \max(S_j) + L\cdot \len(I_j))] \cap [0,L]$\;\label{algoline:unmarked_rangeset}
            }
        } 
        
    }
    \Else { 
        $Y_j := \algmq(I_j, Y_j)$\;
        \If { $\len(Y_j) < \max( 4 L\cdot \len(I_j), 4 L / T )$ }  {
            Bisect $I_j$ into $I_{j1}, I_{j2}$, set $S_{j1} = Y_j, S_{j2} = Y_j$\;
        }
    }
}
\end{algorithm2e}

\begin{algorithm2e}
\caption{Midpoint query procedure: $\algmq(I_j, Y_j)$}
\label{algo:midpt_query}
    Input: Interval $I_j$, associated range $Y_j$\;
    Learner queries $q$, the midpoint of $Y_j$\;
    \If {$\sigma_t$ = 1}{
        Set $Y'_j := [0, q_t + L\cdot\len(I_j)] \cap Y_j$\;
        \szcomment{This and above should be equivalent}
    }
    \Else {
        Set $Y'_j := [q_t - L\cdot\len(I_j), 1] \cap Y_j$\;
        \szcomment{This and above should be equivalent}
    }
    Return $Y'_j$\;
\end{algorithm2e}

\subsection{Algorithm for $C=0$}

It will be helpful to first give a brief description of an algorithm that appeared in~\cite{mao2018contextual}. This algorithm works for $d=1$ and when there are no adversarial corruptions. At each point in time, the learner maintains a partition of the input space into intervals. For each interval $I_j$ in the partition, the learner also maintains an associated range $Y_j$, which serves as an estimate of the image of $I_j$.\szcomment{In particular, the algorithm ensures the following is true: $f(I_j)\subset Y_j, \len(Y_j) = L\cdot O(\len(I_j))$. } When a context appears in $I_j$, the learner selects the midpoint of $Y_j$ as the query point, and the associated range $Y_j$ shrinks and gets refined over time. This corresponds roughly to the subprocedure summarized in \Cref{algo:midpt_query}, which is termed the midpoint query procedure $\algmq$. As shown in \cite{mao2018contextual}, the $\algmq$ procedure preserves the relation $f(I_j)\subset Y_j$ (when there are no corruptions). Moreover, whenever $\len(Y_j) \ge 4L\cdot \len(I_j)$, the associated range $Y_j$ shrinks by at least a constant factor after a single call to $\algmq$ (for more details see \Cref{app:proof1d}). When $Y_j$ reaches a point where significant refinement is no longer possible, the learner zooms in on $I_j$ by bisecting it. 

\subsection{Corruption-Robust Search with Agnostic Checks}

The main algorithm for absolute loss with $d = 1$ is summarized in \Cref{algo:1dabsolute}. This algorithm is corruption-robust and agnostic to the corruption level $C$. Below I give the key ideas in this algorithm. 

When there are adversarial corruptions, it will generally be impossible to tell for certain whether the associated range $Y_j$ contains the image of the interval $I_j$. That is, the learner will not know with complete certainty whether $f(I_j) \subset Y_j$ holds. The analysis divides intervals into three types: \textit{safe intervals}, \textit{correcting intervals}, and \textit{corrupted intervals}. 

\begin{definition} [Corrupted, Correcting, and Safe Intervals]
Consider some interval $I_j$ that occurs during the run of the algorithm. 
\begin{itemize}
\item $I_j$ is a \emph{corrupted interval} if, there exists some $t$ where $x_t\in I_j$ and the signal $\sigma_t$ is corrupted. 
\item $I_j$ is a \emph{correcting interval} if its parent interval is a \emph{corrupted interval} and for every round $t$ where $x_t \in I_j$, the signal $\sigma_t$ is uncorrupted. 
\item $I_j$ is a \emph{safe interval} if its parent interval is safe or correcting (or itself is a root interval), and for every round $t$ where $x_t\in I_j$, the signal $\sigma_t$ is uncorrupted. 
\end{itemize}
\end{definition}

I introduce the \textit{agnostic checking} measure to combat adversarial corruptions. Consider a new interval $I_j$ that has been formed by bisecting its parent interval, and that a context $x_t$ appears in this interval $I_j$. The learner will first query the two endpoints of $Y_j$ and test whether $f(I_j)\subset Y_j$ holds (according to the possibly corrupted signals). The interval passes the agnostic check if according to the (possibly corrupted) signals $f(I_j) \subset Y_j$. If the interval does not pass the agnostic check, it is marked as \emph{dubious}, and the learner resets $Y_j$ to $[0, L]$ and effectively searches from scratch for the associated range. 

The main theorem is stated below. 

\begin{restatable}{theorem}{thmsymmoneD}
\label{thm:symm1D}
\Cref{algo:1dabsolute} incurs $L\cdot O(C \cdot \log T)$ total absolute loss for $d = 1$. 
\end{restatable}

\begin{proof}[Proof Sketch]
The proof consists of several steps. 
\paragraph{Step 1.} An interval not marked \emph{dubious} bisects in $O(1)$ rounds. This is because the associated range is shrinking by at least a constant after each query and after $O(1)$ rounds the associated range will have shrunk enough and meet the criteria for bisecting the interval. By a similar logic, any interval marked \emph{dubious} bisects in $O(\log T)$ rounds. 
\paragraph{Step 2.} Consider any \emph{correcting} interval. After the agnostic checking steps, the associated range must contain the true range of the function on the interval. The associated range will then be accurate when the interval is bisected. Then consider any \emph{safe} interval. Its parent must be safe or correcting, and by induction, a \emph{safe} interval will not be marked \emph{dubious}. 
\paragraph{Step 3.} There can be at most $O(C)$ \emph{corrupted} intervals. Since any \emph{correcting} interval has a \emph{corrupted} interval as a parent, there can be at most $O(C)$ correcting intervals. A total of $O(\log T)$ queries can occur on corrupted and correcting intervals, contributing loss $L\cdot O(C\log T)$ (actually one can show correcting intervals contribute loss at most $L\cdot O(C)$). For \emph{safe} intervals, a loss of magnitude $O(2^{-h})$ can occur at most $O(2^h)$ times (since there are $O(2^h)$ intervals at depth $h$), thus the total loss from safe intervals can be bounded by $L\cdot O(\log T)$. 

Putting everything together gives the total $L\cdot O(C\cdot \log T)$ regret bound. 
\end{proof}

\begin{remark}
Note that in the algorithm, $I_j$ gets bisected when $\len(Y_j)$ reaches below $\max(4L\cdot \len(I_j), 4L/T)$. The latter quantity $4L/T$ ensures that any interval must be bisected in $O(\log T)$ rounds. This is also enough to ensure that all \emph{safe} intervals with depth $>\log T$ contribute total loss at most $O(1)$. 
\end{remark}

\begin{remark}
The regret is tight up to $\log T$ factors. To see this, consider the first $C$ rounds, where the adversary corrupts the signal with probability $1/2$ each round. The learner essentially receives no information during this period, and each round incurs regret $\Omega(L)$. 
\end{remark}


\begin{remark}
    An interesting aspect of this algorithm is that the learner remains agnostic to the type of each interval. In other words, during the run of the algorithm, the learner will in general not know for certain which type an interval belongs to. The analysis makes use of the three types of the interval, not the algorithm itself. 
\end{remark}

\szcomment{Optimal 1d lipschitz? }
\szcomment{1/2, 1/2, 3/4, 3/4, 15/16, 15/16, 35/32, 35/32, 315/256}

\subsection{Extending to $d > 1$}
\Cref{algo:1dabsolute} can be extended in a straightforward way to accommodate the case $d > 1$. The full algorithm and analysis are given in \Cref{app:sym-highd}. The only difference is that the learner maintains $d$-dimensional hypercubes instead of intervals. The main result is as follows. 
\begin{restatable}{theorem}{thmsymmMD}
\label{thm:symmMD}
    \Cref{algo:highDsymm} incurs $L\cdot O_d(C\log T +  T^{(d-1) / d)})$ total absolute loss for $d\ge 2$. 
\end{restatable}

\begin{remark}
In~\cite{mao2018contextual}, it was shown the optimal regret when $C=0$ is $\Omega(T^{(d-1)/(d)})$. Hence, the dependence on $C$ and $T$ are both optimal in the above theorem (up to $\log T$ factors). 
\end{remark}

\section{Algorithm for Pricing Loss}
\label{sec:pricing}

\begin{algorithm2e}[htb]
\caption{Learning with corrupted binary signal under pricing loss}
\label{algo:pricing}
Set parameter $\eta_0 = T^{-1/(d+1)}$, agnostic check schedule $\tau_0$\;
Learner uniformly discretizes input space into hypercubes with lengths $\Theta(\eta_0)$ each\;
For each hypercube $I_j$, learner maintains an associated range $Y_j$ (initialized as $[0, L]$), query count $c_j$ (initialized as $0$)\;
\For {$t = 1, 2, ..., T$} {
    Learner receives context $x_t$\;
    Learner finds hypercube $I_j$ such that $x_t \in I_j$\;
    Let $Y_j$ be the associated range of $I_j$\;
    \If {$\len (Y_j) < 10 L \cdot \eta_0$} { 
        $c_j := c_j + 1$\;
        \If { $c_j > \tau_0$ } {
            Query $\max(Y_j)$\;
            Set $c_j := 0$\;
        }
        \Else { 
            Query $\min(Y_j)$\;
        }
        \If {Learner is surprised (see Definition \ref{def:surprises})}{  
            Set $Y_j := [0,L]$\;
            Set $c_j := 0$\;
        }
    }
    \Else {
        $Y_j := \algmq(I_j, Y_j)$\;
    }
        
}
\end{algorithm2e}

This section discusses the new ideas needed to design corruption-robust algorithm for the pricing loss. The description shall be given in the context of dynamic pricing, and the learner shall be referred to as the seller in this section. The main algorithm is summarized in~\Cref{algo:pricing}. Note that for the pricing loss, the case with $d =1$ and $d > 1$ are treated together. 

Extending~\Cref{algo:1dabsolute} for the absolute loss to pricing loss is not straightforward, since agnostic checks will overprice and the seller necessarily incurs a large loss whenever she overprices. The learner did not have this problem with absolute loss, since the absolute loss is continuous. 

\subsection{Algorithm for $C = 0$}
I first give the description of an algorithm for $C = 0$. The algorithm is adapted from \cite{mao2018contextual} but uses uniform discretization. The input space is discretized into hypercubes with length $\eta_0 = T^{-1/(d+1)}$. When a context appears in some hypercube $I_j$, one of two things can happen: 
\begin{enumerate}
    \item If the associated range is larger than $10L\cdot \eta_0$, the learner performs $\algmq$ and updates the associated range. There can be at most ${O}(\eta_0^d \log T)$ rounds of this type, since for each hypercube the associated range will shrink below $10L\cdot \eta_0$ after $O(\log T)$ queries. 
    \item Otherwise, the associated range is less than $10L\cdot \eta_0$, and the learner can directly set the lower end of the interval as the price. The total loss from these rounds can be bounded as $10L\cdot T\eta_0$. 
\end{enumerate}
The total loss is then $L 
\cdot O(\eta_0^d \log T) + L\cdot O(T\eta_0) = L\cdot {O}(T^{d/(d+1)} \log T)$. 

\szdelete{
The learner searches for the associated range $Y_j$ of a hypercube using the midpoint query procedure $\algmq$ (these shall be termed \emph{searching rounds} or \emph{searching queries}) and stops the search process when the associated range $Y_j$ becomes small enough (specifically, when the length of $Y_j$ drops below $10L\cdot \eta_0$). The hypercube is said to become \emph{pricing-ready} when this happens. When there are no adversarial corruptions, the learner now has a good estimate of the optimal price and sets the lower end of $Y_j$ as the price for any context in this hypercube. These shall be termed \emph{pricing rounds}. At \emph{pricing rounds}, the seller expects the buyer to purchase the item, and assuming the range $Y_j$ is accurate (which is the case when $C = 0$), the revenue loss should be no larger than $10 L \cdot \eta$. The total loss from \emph{pricing rounds} can be bounded by $L\cdot O(T\cdot \eta_0)$. The total loss from searching rounds can be bounded by $L\cdot  \widetilde{O}(\eta_0^{-d})$ since there are $O(\eta_0^{-d})$ hypercubes, and for each hypercube, there can be at most $O(\log T)$ queries before the hypercube becomes \emph{pricing-ready}. By the choice of parameter $\eta_0 = T^{-1/(d+1)}$, the regret bound is $L\cdot \widetilde{O} (T^{d/(d+1)})$. }

\szdelete{
It should be noted that~\cite{mao2018contextual} proposed an algorithm based on adaptive discretization whereas the above algorithm is based on uniform discretization. Using adaptive discretization improves the regret bound by a $O(\log T)$ factor, the presentation for the corruption-robust algorithm for pricing loss is based on uniform discretization, as it highlights the new algorithmic ideas more clearly. Nevertheless, the algorithm can be combined with adaptive discretization by using ideas from the previous section. }

\szcomment{TODO. Compare with adaptive discretization. }\szcomment{DONE}

\subsection{Corruption-Robust Pricing with Scheduled Agnostic Checks}

The seller could potentially run into issues when there are adversarial corruptions. The adversary can manipulate the seller into underpricing by a large margin by only corrupting a small number of signals. Consider the following example. The buyer is willing to pay $0.5$ for an item with context $x$. The adversary can manipulate the signals in the $\algmq$ procedure so that the seller's associated range for $x$ is $[0, \eta]$, where $\eta < 10L\cdot \eta_0$. Hence, the associated range does not contain and is well below the optimal price $0.5$ for this item. The seller then posts a price of $0$ for item $x$. Even though the buyer purchases the item at a price of $0$, the seller is losing 0.5 revenue per round, and this happens with the adversary corrupting only $O(\log T)$ rounds. \szcomment{TODO, rephrase}

\szcomment{ $Y_j$ is not a valid pricing range, and even though she prices the item at the lower end of $Y_j$ and the buyer is purchasing the item, the true valuation could be much higher and the seller is losing a large amount of revenue. }

To combat adversarial corruptions, the learner performs agnostic check queries. These queries differ from agnostic checks for the absolute loss in the following two aspects. First, whereas for the absolute loss, the learner performs agnostic check queries on both ends of the associated range $Y_j$, for the pricing loss the seller only performs agnostic checks on the upper end of $Y_j$. This ensures the seller is not underpricing the product by a large margin. Second, the seller only performs agnostic check queries when the associated range of the hypercube is sufficiently small (below $L\cdot O(\eta_0)$). The seller expects the buyer not the purchase the item during agnostic check queries. 

The learner will set a schedule for performing agnostic checks. At a high level, performing agnostic checks too often incurs large regret from overpricing, while performing agnostic checks too few may not detect corruptions effectively. The schedule for agnostic checks serves as a mean to balance the two. This schedule informs the learner how often to perform agnostic check so as not to incur large regret when overpricing, and at the same time control the loss incurred from corrupted signals. Specifically, the learner keeps track of how many rounds the context vectors arrived in each hypercube. The learner then performs an agnostic check every $\tau_0$ rounds. Here, $\tau_0$ will be a paramter to be chosen later. 

Some notions are introduced. 
\begin{definition} [Pricing-ready hypercubes]
For hypercube $I_j$, if the length of the associated range is below $10L\cdot \eta_0$, then the hypercube is \emph{pricing-ready}. 
\end{definition}

\begin{definition} [Pricing round, checking round, searching round]
If a hypercube is \emph{pricing-ready}, then setting the price as the lower end of the associated range is termed \emph{pricing round}, and setting the price as the upper end of the associated range is termed \emph{checking round}. If a hypercube is not \emph{pricing-ready}, then performing a $\algmq$ is termed \emph{searching round}. 
\end{definition}

\begin{definition}[Surprises]
\label{def:surprises}
    The seller is said to become \emph{surprised} when the signal she receives is inconsistent with her current knowledge. Specifically, the seller becomes surprised when either: she performs a checking round but observes an underprice signal $(\sigma_t = 0)$; or she performs a pricing round but observes an overprice signal $(\sigma_t = 1)$. 
\end{definition}

\begin{definition} [Runs]
Fix some hypercube $I_j$, the set of queries that occurs on $I_j$ before the learner becomes \emph{surprised} (or before the algorithm terminates) called a \emph{run}. 
\end{definition}
When a run ends, the learner resets the associated range and a new run begins. Note there can be multiple runs on the same hypercube.

The algorithm gives the following cumulative loss bound. A detailed analysis can be found in~\Cref{app:pricing}. 
\begin{restatable}{theorem}{thmPricing}
\label{thm:pricing}
Using~\Cref{algo:pricing}, the learner incurs a total pricing loss
\[
L\cdot O(C\log T + T^{d/(d+1)}\log T + C\tau_0 + T/\tau_0)
\]
for any unknown corruption budget $C$. 
Specifically, setting $\tau_0 = T^{1/(d+1)}$, the learner incurs total pricing loss
\[
L\cdot \widetilde{O} (T^{d/(d+1)} + C\cdot T^{1/(d+1)} ). 
\]
\end{restatable}

\begin{proof}[Proof Sketch]
The proof consists of several steps. 
\paragraph{Step 1.} The seller can only become surprised when at least one corruption occurs on the current `run' of some hypercube. Thus the seller can become surprised at most $C$ times. 
\paragraph{Step 2. }For searching rounds, the total loss can be bounded as $L\cdot O( (C + \eta_0^{-d}) \log T)$. There can be at most $O(T/\tau_0)$ checking rounds, contributing loss $L\cdot O(T /\tau_0)$. 
\paragraph{Step 3. } For pricing rounds, the total loss can be bounded as $L\cdot O(C \tau_0 + T \eta_0)$. At a very high level, if the pricing interval is accurate (meaning the lower endpoint is just below the true price by a margin of $L\cdot O(\eta_0)$), then these rounds contribute loss at most $L\cdot O(T \eta_0)$. Otherwise, the seller can detect if the pricing interval is inaccurate by performing agnostic checks, and the loss from these rounds can be bound by $L\cdot O(C \tau_0)$. 

Putting everything together completes the proof. 
\end{proof}


If the corruption budget $C$ is known, the learner can set $\tau_0$ to balance the terms and achieve a sharper regret bound. 


\begin{corollary}
\label{cor:pricingKnownC}
Assume the corruption budget $C$ is known to the learner. Using \Cref{algo:pricing} with $\tau_0 = \sqrt{T/C}$, the learner incurs total loss
\[
L\cdot \widetilde{O} (T^{d/(d+1)} + \sqrt{TC}). 
\]
\end{corollary}

\subsection{Lower Bounds}

This subsection shows two lower bounds for either known or unknown $C$, showing the upper bounds in the previous subsection are essentially tight. Define an environment as the tuple $(f, C, \mathcal{S})$, representing the function, corruption budget, and corruption strategy of the adversary respectively. In the following assume $L = 1$, the scaling to other $L$ is straightforward. 

\begin{theorem}
Let $A$ be any algorithm to which the corruption budget $C$ is unknown. Suppose $A$ achieves a cumulative pricing loss $R(T)$ when $C = 0$, and that $R(T) = o(T)$. Then, there exists some corruption strategy with $C = 2R(T)$, such that the algorithm suffers $\Omega(T)$ regret. 
\end{theorem}
\begin{proof}
Let us consider two environments. In the first environment, the adversary chooses $f(x)\equiv 0.5$ and never corrupts the signal. In the second environment, the adversary chooses $f(x)\equiv 1$ and corrupts the signal whenever the query is above 0.5. 

Now, since the algorithm achieves regret $R(T)$ when $C=0$, then when facing the first environment, the seller can only query values above $0.5$ for at most $2R(T)$ times. Then, by choosing $C = 2R(T)$, the adversary can make the two environments indistinguishable from the seller. Hence the seller necessarily incurs $\Omega(T)$ regret in the second environment. 
\end{proof}
\begin{remark}
The above lower bound shows one cannot hope to achieve a regret bound of the form $O(T^{(d/(d+1))}) + C\cdot o(T^{1/(d+1)})$ when $C$ is unknown. First note any algorithm must achieve $\Omega(T^{d/(d+1)})$ in the worst case when $C=0$ by the lower bound in \cite{mao2018contextual}. Then, suppose the algorithm could achieve $O(T^{d/(d+1)})$ when $C=0$, the algorithm must achieve $\Omega(T)$ regret for some $C = \Theta(T^{d/(d+1)})$. 
\end{remark}


Next, a lower bound is given when the corruption budget $C$ is {known}. The lower bound will be against a randomized adversary, who draws an environment from some probability distribution $\cD$, and the corruption budget is defined as the expected value of $C$ under distribution $\cD$. Note that the upper bounds (\Cref{thm:pricing}, Corollary \ref{cor:pricingKnownC}) also hold for randomized adversaries. 
\begin{theorem}
There exists some $\cD$ where any algorithm incurs expected regret $\Omega(\sqrt{CT})$, even with complete knowledge of $\cD$. 
\end{theorem}
\begin{proof}
The adversary uses two environments. In the first environment, the function $f(x) \equiv 0.5$ and the corruption budget is 0. In the second environment, the function $f(x) \equiv 1$, the corruption budget is $C_0$, and the adversary corrupts any query above 0.5 until the budget is depleted. The adversary chooses the first environment with probability $1-p$ and the second environment with probability $p$. 

If the learner's algorithm queried more than $C_0$ rounds above $0.5$, then the learner incurs regret $0.5C_0$ in the first environment, giving expected regret $\Omega((1-p) C_0)$. If the algorithm did not query more than $C_0$ rounds above $0.5$, then the algorithm incurs regret $0.5(T - C_0)$ in the second environment, giving expected regret $\Omega(p(T-C_0))$. Choosing $p = \sqrt{C/T}, C_0 = \sqrt{CT}$ completes the proof. 
\end{proof}





	


\section{Conclusion and Future Work}
In this work, I design corruption-robust algorithms for the Lipschitz contextual search problem. I present the \emph{agnostic checking} technique and demonstrate its effectiveness in designing corruption-robust algorithms. An open problem is closing the gap between upper bounds and lower bounds, in particular for the absolute loss when $d = 1$. Specifically, can one actually achieve $O(C + \log T)$ regret in this setting? Another interesting future direction is to relax the Lipschitzness assumption. For example, this work assumes the learner knows the Lipschitz constant $L$. Can the learner design efficient no-regret algorithms without knowledge of $L$? 

\acks{The author would like to thank Jialu Li for her help in preparation of this manuscript. The anonymous reviewers at ALT gave helpful suggestions on improving the presentation of this paper. }

\bibliography{references}

\appendix


\section{Proof Details for Absolute Loss}
\label[appendix]{app:proof1d}

\subsection{Absolute Loss with $d = 1$}



An interval has depth $r$ if it was bisected from an interval at depth $r-1$. Initial root intervals have depth 0. Hence, an interval $I_j$ at depth $r$ has length $\len(I_j) = O(2^{-r})$. 

The below two lemmas discuss properties of $\algmq$, and are adapted from \cite{mao2018contextual}. The below lemma shows that each call to $\algmq$ preserves the property $f(I_j)\subset Y_j$, assuming the signal is uncorrupted. 
\begin{lemma}
\label{lemma:updateValid}
    Consider a call to the $\algmq(I_j, Y_j)$ procedure. Let $Y_j$ be the associated range before the query, and let $Y'_j = \algmq(I_j, Y_j)$ be the range after the query. If the signal was uncorrupted and $f(I_j) \subset Y_j$, then $f(I_j) \subset Y'_j$. 
\end{lemma}

\begin{proof}
Let $q_t$ be the query point (i.e. midpoint of $Y_j$). Assume the signal $\sigma_t = 0$, so that $f(x_t) \ge q_t$. The case where $\sigma_t = 1$ will be similar. 

Since $f$ is $L$-Lipschitz, it must be that 
\[
f(I_j) \ge f(x_t) - L\cdot \len(I_j) \ge q_t - L\cdot \len(I_j), 
\]
hence the update 
\[
Y'_j = Y_j\cap [q_t - L\cdot \len(I_j) , L]
\]
still guarantees that $f(I_j) \subset Y'_j$. 
\end{proof}

The next lemma shows the length of $Y_j$ shrinks by a constant factor after each call to $\algmq$, regardless of whether the signal was corrupted or not. 
\begin{lemma}
\label{lemma:rangeShrinks}
Consider a call to the $\algmq(I_j, Y_j)$ procedure. Let $Y_j$ be the associated range before the query, and let $Y'_j = \algmq(I_j, Y_j)$ be the range after the query. Suppose $\len(Y_j) \ge 4L\cdot \len(I_j)$, then $\frac{1}{2} \len{(Y_j)} \le \len{(Y'_j)} \le \frac{3}{4} \len{(Y_j)}$ regardless of whether the signal is corrupted or not. 
\end{lemma}

\begin{proof}
    Let $q_t$ be the query point (i.e. midpoint of $Y_j$). Assume the signal $\sigma_t = 0$, the case with $\sigma_t = 1$ will be similar. 
    
    Before the call to $\algmq(I_j, Y_j)$, the learner has $\len(Y_j) \ge 4L\cdot \len(I_j)$. This implies 
    \[
    q_t - \min(Y_j) = \frac{\len(Y_j)}{2} \ge 2L\cdot \len(I_j)
    \]
    so that
    \[
    q_t - L\cdot \len(I_j) > \min(Y_j). 
    \]

    The update can be written as
    \begin{align*}
        Y'_j &= [q_t - L\cdot \len(I_j), L] \cap Y_j\\
        &= [q_t - L\cdot \len(I_j), q_t] \cup [q_t, \max(Y_j)]
    \end{align*}
    The lemma then follows from $\len{( [q_t, \max(Y_j)] )} = \frac{1}{2} \len{ ( Y_j ) }$, $\len{ ( [q_t - L\cdot \len(I_j), q_t] ) } = L\cdot \len(I_j) \le \frac{1}{4} \len{ ( Y_j ) }$. 
\end{proof}

I now give several lemmas that discuss various properties of the three types of intervals (\emph{safe}, \emph{correcting}, or \emph{corrupted}). 
\begin{lemma}
\label{lemma:bisectFast}
    If an interval is marked dubious, then the interval is bisected after $O(\log T)$ queries. If an interval is not marked dubious, then the interval is bisected after $O(1)$ queries. 
\end{lemma}

\begin{proof}
    If the interval is marked dubious, then the range of the interval is reset to $[0,L]$. By Lemma~\ref{lemma:rangeShrinks}, the range $Y_j$ shrinks by at least a factor of $\frac{3}{4}$ each query, hence after $O(\log T)$ queries, the learner has $\len(Y_j) < \max(4L\cdot \len(I_j), 4 L / T) $. 

    If the interval were not marked dubious, then the range $Y_j$ is updated as
    \[
    Y_j = [\min(S_j) - L\cdot \len(I_j), \max(S_j) + L\cdot \len(I_j)] \cap [0,L]. 
    \]
    From the above, it can be seen that
    \[
    \len(Y_j) \le \len(S_j) + 2L\cdot \len(I_j). 
    \]
    By the splitting condition,
    \[
    \len(S_j) \le \max(8L\cdot \len(I_j), 4L / T). 
    \]
    If the interval $I_j$ has length larger than $1/(2T)$, then $\len(S_j) \le 8 L \cdot \len(I_j) $, and after the update the learner has $\len(Y_j) \le 10 L \cdot \len(I_j)$. If the interval $I_j$ has length no larger than $1/(2T)$, then $\len(S_j) \le 4 L / T$,  and after the update $\len(Y_j) \le 5 L / T$. In either case, Invoking Lemma~\ref{lemma:rangeShrinks} again, the range shrinks by constant factor each query and it takes $O(1)$ rounds for the range to shrink below $\max ( 4L \cdot \len(I_j), 4 L / T)$. 
\end{proof}

\begin{lemma}
\label{lemma:amendingFixes}
    Let $I_j$ be a correcting interval. When $I_j$ is bisected, it holds that $f(I_j) \subset Y_j$. 
\end{lemma}

\begin{proof}
    It is first shown that $f(I_j) \subset Y_j$ after the update following the agnostic check procedure. There are two cases to consider, whether $I_j$ has been marked dubious or not after the agnostic check. 
    
    If $I_j$ has been marked dubious, then the update resets $Y_j$ to $[0,L]$, hence the learner has $f(I_j) \subset Y_j$ trivally. 
    

    If $I_j$ has not been marked dubious, the update
    \[
    Y_j = [\min(S_j) - L\cdot\len(I_j), \max(S_j) + L\cdot\len(I_j)]
    \]
    takes place. Since $I_j$ is a correcting interval, the signals are uncorrupted, and by Lipschitzness of $f$, the learner has
    \begin{align*}
        f(I_j) \ge \min(S_j) - L\cdot\len(I_j) \\
        f(I_j) \le \max(S_j) + L\cdot\len(I_j). 
    \end{align*}
    Hence the learner has $f(I_j) \subset Y_j$ after the update. 
        
    Putting these two cases together, after the update on the associated range following the agnostic checking steps, the learner has $f(I_j) \subset Y_j$. By inducting on Lemma~\ref{lemma:updateValid}, repeated calls to $\algmq$ guarantees $f(I_j) \subset Y_j$ when $I_j$ is bisected. 
\end{proof}

\begin{lemma}
\label{lemma:safeNotMarked}
    Safe intervals are not marked dubious. 
\end{lemma}

\begin{proof}
Induction is used to show the following for any safe interval $I_j$: 
\begin{enumerate}
    \item $f(I_j) \subset S_j$
    \item $f(I_j) \subset Y_j$ when $I_j$ is bisected
\end{enumerate}

Let $I_j$ be a safe interval. If $I_j$ is a root interval, then trivially $f(I_j)\subset S_j$. Further by Lemma~\ref{lemma:updateValid}, $f(I_j) \subset Y_j$ when $I_j$ is split. 

If $I_j$ has a correcting interval as its parent interval, then by Lemma~\ref{lemma:amendingFixes}, $f(I_j) \subset S_j$, and consequently by Lemma~\ref{lemma:updateValid} $f(I_j) \subset Y_j$ when $I_j$ is split. 

If $I_j$ has a safe interval as its parent interval, then a simple induction shows the desired result. 
\end{proof}

The following corollary follows directly from Lemma~\ref{lemma:bisectFast} and Lemma~\ref{lemma:safeNotMarked}. 
\begin{corollary}
\label{cor:safeBisectFast}
Safe intervals bisect in $O(1)$ rounds. 
\end{corollary}




\begin{lemma}
\label{lemma:corruptedLoss}
    Corrupted intervals contribute $L\cdot O(C \log T)$ loss. 
\end{lemma}
\begin{proof}
    There are at most $C$ corrupted intervals, and each interval splits in $O(\log T)$ rounds by Lemma~\ref{lemma:bisectFast}, with each round incurring $O(L)$ regret (trivially). Hence corrupted intervals contribute total $L \cdot O(C\log T)$ regret. 
\end{proof}

\begin{lemma}
\label{lemma:amendingLoss}
    Correcting intervals contribute $L\cdot O(C)$ loss. 
\end{lemma}
\begin{proof}
    There are at most $2C$ correcting intervals. For each correcting interval $I_j$, querying the two endpoints of $S_j$ incur $O(L)$ regret. Each call to $\algmq(I_j, Y_j)$ incur $O(\len(Y_j))$ regret. By Lemma~\ref{lemma:rangeShrinks}, $\len(Y_j)$ shrinks by a factor of $\frac{3}{4}$, hence the loss is geometrically decreasing with each query, and the total loss incurred within $I_j$ is $O(L)$. 
\end{proof}

The below lemma bounds the regret from safe intervals, and is adapted from \cite{mao2018contextual}. 
\begin{lemma}
\label{lemma:safeLoss}
Safe intervals contribute $L\cdot O(\log T)$ loss. 
\end{lemma}

\begin{proof}
    Consider a safe interval $I_j$ at depth $h < \log T$. The two queries on the endpoints incur $O(2^{-h})$ loss, since 
    \[
    \len(S_j) < 8 L \cdot \len(I_j) = O(2^{-h}). 
    \]
    Each call to $\algmq(I_j, Y_j)$ also incurs $L\cdot O(2^{-h})$ loss since $\len(Y_j) = O(L\cdot \len(I_j)) = L\cdot O(2^{-h})$. By Corollary~\ref{cor:safeBisectFast}, safe intervals splits in $O(1)$ rounds, hence the loss incurred at interval $I_j$ is $L\cdot 
 O(2^{-h})$. Note that there can be at most $O(2^h)$ intervals at depth $h$, therefore a loss of magnitude $L\cdot O(2^{-h})$ can be charged at most $O(2^h)$ times. In a total of $T$ rounds, the loss can be at most $L\cdot O(\log T)$. 
    
    
\end{proof}

\thmsymmoneD*
\begin{proof} Summing over the loss contributed by corrupted intervals (Lemma \ref{lemma:corruptedLoss}), correcting intervals (Lemma \ref{lemma:amendingLoss}), and safe intervals (Lemma \ref{lemma:safeLoss}) completes the proof. 
\end{proof}

\subsection{Absolute Loss with $d > 1$}
\label[appendix]{app:sym-highd}

\begin{algorithm2e}
\caption{Learning with corrupted binary signal under absolute loss for $d > 1$}
\label{algo:highDsymm}
Learner maintains a partition of hypercubes $I_j$ of the input space $[0,1]^d$ throughout learning process\;
For each hypercube $I_j$, learner stores a checking interval $S_j$ and maintains an associated range $Y_j$\;
The partition is initialized as a single hypercube with checking interval $S_j$ set to $[0, 8L]$\;
\For {$t = 1, 2, ..., T$} {
    Learner receives context $x_t$\;
    Learner finds hypercube $I_j$ such that $x_t \in I_j$\;
    Let $Y_j$ be the feasible interval of $I_j$\;
    \If {Exists an endpoint of $S_j$ not queried} {
        Query an unqueried endpoint of $S_j$\;
        \If {Queried $\min(S_j)$ and $\sigma_t = 1$, or queried $\max(S_j)$ and $\sigma_t = 0$} {
        Mark $I_j$ as dubious\;%
        }
        \If {Both endpoints of $S_j$ have been queried}{
            \If {$I_j$ marked dubious}{
                Set range $Y_j := [0, 8L]$\;\label{algoline:highd_marked_rangeset}
            }
            \Else{
                Set range $Y_j = [\min(S_j) - L\cdot \len(I_j)), \max(S_j) + L\cdot \len(I_j))] \cap [0,1]$\; \label{algoline:highd_unmarked_rangeset}
            }
        }
    }
    \Else {
         $Y_j := \algmq(I_j, Y_j)$\;
        \If { $\len(Y_j) < \max (4 L\cdot \len(I_j), 4L / T )$ } { 
            Bisect each side of $I_j$ to form $2^d$ new hypercubes each with length $\len(I_j) / 2$\;
            For each new hypercube $I_{ji}$ ($1\le i \le 2^d$) set $S_{ji} = Y_j$\;
        } 
        
    } 
}
\end{algorithm2e}

The corruption-robust algorithm for learning a Lipschitz function from $\cR^d \rightarrow \cR (d > 1)$ under absolute loss is summarized in~\Cref{algo:highDsymm}. The proposed algorithm initializes the input space as a single hypercube with $S_j$ initialized to $[0, 8L]$ (instead of $8^d$ hypercubes with $S_j$ initialized as $[0,L]$). The two initializations are essentially equivalent but the former slightly simplifies analysis somewhat: there is now only a single root hypercube. 

\thmsymmMD*
The analysis will be similar to that of~\Cref{algo:1dabsolute} and~\Cref{thm:symm1D}. 

\begin{proof}
There are $O(C)$ corrupted intervals, contributing $L\cdot O(C\log T)$ loss. There are $O(C\cdot 2^d)$ correcting intervals, contributing $ L\cdot O( C\cdot 2^d)$ loss. For safe intervals, there are at most $O(2^{hd})$ intervals at depth $h$. Throughout $T$ rounds, a loss with magnitude $L\cdot O(2^{-h})$ can be charged at most $O(2^{hd})$ times, and the total loss of all safe intervals is then at most $L\cdot O(T^{(d-1) / d})$, since the loss coming from safe intervals can be upper bounded by
\begin{align*}
    L\cdot \sum_{h=0}^{\frac{\log T}{d}} O( 2^{(d-1)h} ) = L\cdot O(T^{{(d-1)}/{d}}). 
\end{align*}
Putting the loss of all three types of intervals, the total loss is
\[
L\cdot O(C\log T + C\cdot 2^d + T^{(d-1)/{d}}) = L\cdot O_d(C\log T + T^{(d-1)/d}),
\]
which completes the proof. 
\end{proof}

\section{Proof Details for Pricing Loss}
\label{app:pricing}
\szcomment{Don't use adaptive zooming / adaptive discretization, use uniform discretization}

Recall that on a pricing-ready hypercube $I_j$, the learner is said to become \emph{surprised} when the signal is inconsistent with $Y_j$. In other words, the learner becomes surprised when she queried $\min(Y_j)$ and receives an overprice signal ($\sigma_t = 1$), or queried $\max(Y_j)$ and receives an underprice signal ($\sigma_t = 0$). 

\szdelete{
The sanity check schedule ensures the following holds. 
\begin{claim}
Let a pricing interval be run for $t_0$ rounds before the learner becomes surprised or the algorithm terminates, then the number of sanity checks is $\Theta(g^{-1}(t_0))$. Conversely, if there are $s_0$ sanity checks in this run, then the interval has been run for $\Theta(g(s_0))$ rounds. 
\end{claim}
}

\begin{lemma}
\label{lemma:surprised}
    The learner becomes surprised at most $C$ times. 
\end{lemma}

\begin{proof}
    The learner can become surprised for the following two reasons: the signal itself is corrupted, or the signal is uncorrupted, but the associated range $Y_j$ is inaccurate ($f(I_j)\not\subset Y_j$). In the latter case, a corruption must have occurred in the search queries before the hypercube became pricing ready. Since there are at most $C$ corruptions, the learner can become surprised at most $C$ times. 
\end{proof}

\begin{lemma}
\label{lemma:pricingLoss}
    Consider a run of length $\zeta$ on some hypercube $I_j$. Further, assume a total of $\xi$ signals were corrupted during this period. Then pricing rounds pick up a loss of $L\cdot O(\zeta \eta_0 + \xi \tau_0)$ during this run. 
\end{lemma}

\begin{proof}
First note that once the hypercube becomes pricing-ready, the associated range $Y_j$ does not change until this run terminates. If $f(I_j) \subset Y_j$, then each pricing round incurs loss $L\cdot O(\eta_0)$, hence a run with length $\zeta$ incurs loss $L\cdot O(\zeta \eta_0)$. 

Otherwise, there are four cases. 
\begin{enumerate}
    \item $f(I_j) \ge \max(Y_j)$
    \item $f(I_j)$ has some overlap with $Y_j$ and $f(I_j) \ge \min(Y_j)$, $\max(Y_j) \in f(I_j)$
    \item $f(I_j) \le \min(Y_j)$
    \item $f(I_j)$ has some overlap with $Y_j$ and $f(I_j) \le \max(Y_j)$, $\min(Y_j) \in f(I_j)$
\end{enumerate} 
\szcomment{Insert a illustration here. }

For case 1, any uncorrupted signal in agnostic check queries makes the learner surprised. Hence assuming the adversary spent corruption budge $\xi$, then the run must terminate in $O( \xi \tau_0 )$ rounds since the adversary must be corrupting every agnostic check query. The learner then trivially incurs loss $L\cdot O(\xi \tau_0)$. 

For case 2, the loss collected is at most $L\cdot O(\eta_0)$ per pricing round, hence the total loss is $L\cdot O(\zeta \eta_0)$. 

For case 3, any uncorrupted signal in pricing rounds makes the learner surprised. Hence assuming the adversary spent a corruption budget $\xi$, the run terminates in $O(\xi)$ rounds (since the adversary must be corrupting all pricing rounds), and the learner incurs regret $L\cdot O(\xi)$. 

For case 4, the adversary does not corrupt agnostic check rounds, or the learner will become immediately surprised. If the learner does not become surprised during pricing rounds, this could be due to either: 1. the learner did not overprice, or 2. the learner overpriced but the adversary corrupted the signal. Consequently in pricing rounds, uncorrupted queries accumulate $L\cdot O(\zeta \eta_0)$ loss, and corrupted queries accumulate $L\cdot O(\xi)$ loss. 

Putting all cases together completes the proof. 
\end{proof}

In the following, let $\cT$ be a multi-set with elements denoting the length of runs of pricing-ready hypercubes. Let $\cC$ be a multi-set with elements denoting the number of corruptions that occurred in runs of pricing-ready hypercubes. 

\thmPricing*
\begin{proof}
    First, an upper bound is obtained on the total loss incurred from searching rounds (i.e., when $\len(Y_j) > 10L\cdot \eta_0$). The learner searches from scratch when starting from initialization ($Y_j$ is set to $[0, L]$ at initialization) or when she becomes surprised ($Y_j$ is reset to $[0, L]$). This can happen at most $C + \eta_0^{-d}$ times, since the learner becomes surprised at most $C$ times (Lemma~\ref{lemma:surprised}) and there are $\eta_0^{-d}$ hypercubes at initialization.  Whenever the learner searches from scratch, it takes $O(\log T)$ queries before the interval becomes pricing-ready. Thus the total loss incurred from search queries is 
    \begin{align}
    L \cdot ( (C + \eta_0^{-d}) \log T). 
    \end{align}

    Next, the loss incurred when range becomes pricing-ready can be separated into two parts, loss from agnostic check queries and loss from pricing rounds. 
    The total loss from agnostic check queries can be bounded by
    \begin{align}
    L\cdot O ( T / \tau_0 ), 
    \end{align}
    since there can be $O(T/\tau_0)$ agnostic check queries, and each query contribute loss at most $L$. 
    
    From Lemma~\ref{lemma:pricingLoss}, the total loss from pricing rounds can be bounded by
    \begin{align}
    L\cdot O\left( \sum_{\xi \in \cC} \xi \cdot \tau_0 + \sum_{\zeta\in \cT} \zeta \cdot \eta_0 \right) &\le L\cdot O( C \tau_0 +  T \eta_0 ). 
    \end{align}

    Putting everything together, the total loss can be bounded by the sum of loss incurred during agnostic checking rounds, pricing rounds, and searching rounds:
    \[
    L\cdot {O}\left( T / \tau_0 + C\tau_0 + T\eta_0 + \eta_0^{-d} \log T + C\log T \right). 
    \]
    Plugging the choice of parameter $\eta_0 = T^{-1/(d+1)}$ finishes the proof. 
\end{proof}




\end{document}